\documentclass{article}

\usepackage{graphicx}
\newtheorem{definition}{Definition}
\newtheorem{theorem}{Theorem}
\newtheorem{example}{Example}
\newtheorem{corollary}{Corollary}
\newtheorem{proof}{Proof}
\newtheorem{remark}{Remark}

\begin{document}


\title{Machine Learning via rough mereology}
\author{Lech T. Polkowski\\ University of Warmia nad Mazury in Olsztyn\\ S\l oneczna str.54, 10-710 Olsztyn, Poland\\e-mail: polkow@pjwstk.edu.pl}
\date{}

\maketitle


\begin{abstract}\  Rough sets (RS) proved a thriving realm with successes in many fields of ML and AI. In this note, we expand RS to RM - rough mereology which provides a measurable degree of uncertainty to areas listed in Highlights. We propose an insight into this topic.
\end{abstract}


\section{Rough Sets}

 Rough sets see \cite{Paw}, \cite{RSMF} were investigated from many angles which we outline shortly.

 Among algebraic structures investigated were  Nelson and Heyting algebras, Stone and dual Stone algebras,  Brouwer-Zadeh lattices, \L ukasiewicz and Wajsberg algebras see \cite{RSMF}.\  Consult \cite{RSMF} for topological theory of rough set spaces. For logics  in rough set structures see \cite{RSMF}.

The problem of uncertain knowledge is critical for rough sets and it  permeates  most research  but here we are able to mention only a few. The reader may consult \cite{Paw}, \cite{Polk4}, \cite{NguSk}  \cite{SNY},  \cite{Slez}.\

Rough set research is occupied as well with knowledge reduction, see \cite{NguSk}, \cite{SkoR}. \

RS submits many examples of contexts where the part - whole language comes in a natural way. Henceforth, the 2nd sect. is on this theory - Mereology.

\section{Mereology} Mereology  \cite{Les1}, \cite{Polk3} treats entities as wholes composed of parts. The first step is Ontology in \cite{Les2} where the basic notion of an individual entity is defined.\

The {\em Le\'{s}niewski Axiom of Ontology} consists of three parts joined by the copula {\bf is}, {\bf is} may be replaced with $\in$:

$x\ {\bf is}\ Y\equiv (\exists y. y\ {\bf is}\ x) \wedge (\forall w, z. (w\ {\bf is}\ x)\wedge (z\ {\bf is}\ x) \supset (w\ {\bf is}\ z)) \wedge (\forall y. (y\ {\bf is}\ x) \supset$  $ (y\ {\bf is}\ Y)).$

In the sequel, we discern between the symbolic theory and its denotation in models.  Mereology is the theory of the predicate {\em part of ...} denoted by the symbol $part$. The predicate $part$ is interpreted in domains consisting of individual entities as a relation of a part $Prt$. The atomic formula $part(a,b)$ reads {\em   a is a part of b}.\ The predicate $part$ has to fulfill the following axiom schemes. The symbol $\bot$ denotes falsity.

\begin{definition} [Axiom schemes for part predicate]\label{part} These axiom schemes are the following.\end{definition}

(A1)\ $part(x,y)\wedge part(y,x)\supset \bot$: the relation $Prt$ is anti-symmetric.

(A2)\ $part(x,y)\wedge part(y,z)\supset part(x,z)$: the relation  $Prt$ is transitive.

A corollary follows:

(A3)\ $part(x,x) \equiv \bot$: no $x$ is a part of itself. This follows from (A1): the relation $Prt$  is anti-reflexive.\

The predicate $part$ is interpreted by (A1) as the  relation $Prt$ of a proper part. To provide for the relation of being a whole of an entity, we extend $part$ to the predicate $subst$ called the subset.  In domains of individuals, we name the  relation denoting $subst$ as the complement and we write it down as  $Cmp$.
\begin{definition}[The subset predicate]\label{subset}  $subst(x,y)\equiv part(x,y)\vee (x=y)$.\end{definition}

\begin{theorem}\label{subset 1}\ The following are properties of the predicate $subst$.\end{theorem}

(B1)\ $subst(x,x)$.

(B2)\ $subst(x,y)\wedge subst(y,x)\supset (x=y)$.

(B3)\ $subst(x,y)\wedge subst(y,z)\supset subst(x,z)$.\

We introduce the predicate of overlap $Ov$ which interprets in its domain as the relation  $Ovlp$.

\begin{definition}[The predicate of Overlap]\label{over4}\ $Ov(x,y)\equiv \exists z.(subst(z,x)\wedge$

\noindent $subst(z,y)).$\end{definition}

The relation of overlap $Ovlp(x,y)$ holds if and only if: $Ovlp(x,y)\equiv$

\noindent $\exists z. Cmp(z,x)\wedge Cmp(z,y)$.\

The axiom scheme (A4)  provides bond between predicates $subst$ and $Ov$.

\begin{definition} [Axiom scheme (A4)]\label{a4}

$(A4)\ subst(x,y)\equiv \forall z.(subst(z,x)\supset\exists w.subst(w,y)\wedge O(z,w))$.
\end{definition}

This scheme allows for verification whether $subst(x,y)$ holds on the basis of fulfillment of the right hand side.\

The crucial notion of a class as an idividual which represents a collection of other individuals is defined by means of the following postulates \cite{Les1}.

\begin{definition}[Class]\label{class}  For a non-empty family ${\mathcal{F}}$ of entities,  an entity  denoted $Cls({\mathcal{F}})$ is postulated as satisfying the following. \end{definition}

(C1)\ If $x \in {\mathcal{F}}$, then $subst(x, Cls({\mathcal{F}})$.

(C2)\ $\forall x. [subst(x, Cls(\mathcal{F})\supset \forall y. (subst(y,x)\supset\exists w\in {\mathcal{F}}. Ov(y,w))]$.

The  following axiom scheme secures existence of classes \cite{Les1}.

(A5)\ Cls(${\mathcal{F}}$) exists for each non-empty collection of entities in the given domain. By (A4) and (A5), the class is unique for any given non-empty collection of entities.

Mereology allows for quasi-topological predicates of exteriority $ext$ (as the relation: $Ext$), universe $U$ (as the relation: $V$), topological complement $TCmp$ (as the relation: $TCMP$), and, relative topological complement $TRCmp$(as the relation: $TRCmp$). Here are definitions.
\begin{definition} [Topological predicates]\label{topomereo}\ They are  the following.\end{definition}

(E)\ $ext(x,y)\equiv \neg Ov(x,y)$: $x$ and $y$ are exterior to each other.

(V) there exists the universal class $V=Cls(\{x: x = x\}).$

(C)\ for each $x$, there exists the topological complement $-x=Cls(\{y: ext(y,x)\})$.

(R)\ for entities $x,y$ with $part(x,y)$, there exists relative topological  complement $z=Cls(\{t: part(t,y)\wedge ext(t,x)\})$.\

We are ready for algebraic structure on entities.\  We define, after  \cite{Tar1}, a Boolean algebra $\mathcal{B}$.
\begin{definition} [The Tarski algebra]\label{tarski}\ This algebra is defined as follows.\end{definition}

\begin{enumerate}
\item\ $x+y=Cls(\{t:subst(t,x)\vee subst(t,y)\})$.
\item\ $x \cdot y = Cls(\{t:subst(t,x)\wedge subst(t,y)\})$.
\item\ $-x$ is already defined in (C) above.
\item\ {\bf 1} = $V$.
\item\ Completeness is secured by the class operator.
\end{enumerate}
By (B1), $\mathcal{B}$ has no zero element. Containment expressed as part predicate or $Cmp$ relation is deterministic. Uncertain reasoning needs the notion of a containment to a degree and Rough Mereology supplies it.

\section{Rough Mereology and rough subsets} Partial containment arises naturally in problems related to knowledge systems. One may mention here, the fuzzy  approach  \cite{Za} with fuzzy membership functions,  \cite{PS1} with rough membership functions.  For more information please consult \cite{Polk3}. We add  to our tools an implication.
\begin{definition} [The mereological implication]\label{impl} For entities $x,y$, we let  $$x\rightarrow y \equiv -x+y$$\end{definition}
The validity condition for $x\rightarrow y$ is  $-x+y=V$.\

We denote by the symbol $\Rightarrow$ the relation which interprets $\rightarrow$ in domains of entities. The implication $\rightarrow$ is related to operations in the  algebra $\mathcal{B}$.
\begin{theorem} [Mereological implication]\label{mimp}\ The following are equivalent.\end{theorem}
\begin{enumerate}
\leftskip=2.3mm \labelsep=3.2mm
\item\  $subst(x,y)$.
\item\  $x\cdot y=x$
\item\  $x\rightarrow y$ is valid.
\end{enumerate}
We define the notion of weight.

\begin{definition}[Weights]\label{mass} Given a mereological structure  $(\Omega, part)$ over the relational vocabulary $\{part\}$, we define a weight $w$ by means of  conditions

(W1)\  $\forall x\in \Omega. w(x)\in (0,1]$.

We introduce  the empty entity $\emptyset \notin \Omega$ with weight $0$.

(W2)\  $w(\emptyset)=0$.

(W3)\  $(x\rightarrow y)\supset [w(y)=w(x)+w((-x)\cdot y)].$\end{definition}
\begin{theorem}\label{stat}
The following properties result from axiom schemes (W1)-(W3).
\end{theorem}

\begin{enumerate}
\leftskip=2.3mm \labelsep=2mm \itemsep=1.2pt
\item\  $subst(x,y)\equiv x\cdot y=x$.\

\item\  $subst(x,y)\equiv x\rightarrow y$.\

\item\  $(susbt(x, y)\supset (w(x)\leq w(y))$.\

\item\  $w(x+y)=w(x)+w((-x)\cdot y)$.\

\item\  $(x\cdot y=\emptyset)\equiv w(x+y)=w(x)+w(y)$.\

\item\  $w(x)+w(-x)=1$.\

\item\  $w(y)=w(x\cdot y)+w((-x)\cdot y)$.\

\item\  $subst(x,y)\equiv w(x\rightarrow y)=1$.\

\item\  $w(x\rightarrow y)= 1-w(x-y)$.\vspace{0.5mm}
\end{enumerate}
We now introduce the notion of a rough subset. Rough mereology is the theory of rough subsets, see  \cite{Polk3}. The relation of rough subset denoted $rsubst_w$ for the given weight $w$  is defined as follows.
\begin{definition} [Rough subset]\label{ri} \end{definition}

(rs1)\ $rsubst_w(x,y, r)\equiv \frac{w(x\cdot y)}{w(x)}\geq r$.

 The function {\em maximal rough subset} $rs_w^*$ is defined as follows.

(rs2)\  $rs_w^*(x,y) = argmax_r rsubst_w(x,y,r)=\frac{w(x\cdot y)}{w(x)}$.

\begin{theorem} [Properties of rough subsets]\label{rset}
Axiom schemes  (W1)-(W4) and Defs. (rs1)-(rs2) imply the following.\end{theorem}

\begin{enumerate}
\leftskip=2.3mm \labelsep=2mm
\itemsep=2.6pt
\item\ $subst(x,y)\equiv rs_w^*(x,y)=1$.
\item\ $subst(x,y)\equiv rs_w^*(x,y)=1 \equiv x\rightarrow y=V$.
\item\ $(rs_w^*(x,y)=1)\wedge rs_w^*(z,x,r)\supset rsubst(z,y,r)$.
\item\ $rs_w^*(x,-y)=1-rs_w^*(x,y)$.
\end{enumerate}
For practical purposes, we mention some  means by which one may produce rough subsets. For t-norms,  Archimedean t-norms, and their residual implications, see \cite{Polk4}. In particular, we recall the Hilbert style decompoosition of the Archimedean rough inclusions:
$$(res)\ T(x,y)=h(g(x)+g(y)),$$ where $h,g$ are mappings on the interval $[0,1]$ into $[0,1]$.
\begin{theorem}[Rough subsets from t-norms]\label{t} For any t-norm $T$, the formula $$rsubst_T(x,y,r)\equiv x\Rightarrow_Ty\geq r,$$ where $\Rightarrow_T$ is $T$ - induced residual implication, defines a rough subset relation\ (ii)\ In case t-norm $T$ is Archimedean, with decomposition $T(x,y)=h(g(x)+g(y))$ cf.  \cite{Polk4}, the formula $$rsubst_T(x,y)\equiv h(|x-y|)\geq r$$ defines the rough subset relation.\end{theorem}
\begin{proof} It follows from basic properties of t-norms and their residua; eventually, please see proof in \cite{Polk4}.\end{proof}
We recall the notion of an information system cf. \cite{Paw} as a relational system  $(U, F, V)$ with the set $F$ of features, $V$ a set of feature values, over a domain $U$.  We let $Dis(x,y) = \{f\in F: f(x)\neq f(y)\}$ and $Ind(x,y) = U\times U\setminus Dis(x,y)$.
\begin{theorem} [Archimedean rough subsets in information systems]\label{arch}\ For an Archimedean t-norm $T$, and an information system $I$, we let $$rsubst_{T}^I(x,y,r)\equiv h(\frac{|Dis(x,y)|}{|F|})\geq r$$ which defines a rough subset relation. Specifically, for the \L ukasiewicz t-norm $L(x,y)=max\{0, x+y-1\}$, with$ h(x)=1-x$, we obtain $$rsubst_{T}^I(x,y,r)\equiv \frac{|Ind(x,y)|}{|F|}\geq r$$ \end{theorem}
\begin{proof} See \cite{Polk4}. \end{proof}
Rough subsets form a bridge between rough and fuzzy set theories. We give some justification for this claim.
\begin{theorem}[Fuzzy partitions from rough subsets]\label{fuzz} For any $y\in [0,1]$ and a rough subset $rsubst_w$, we let $rsubst_{w,y}(x)=r \equiv rsubst_w(x,y,r)$.  The symmetric form is $\vartheta(x,y) \equiv  rsubst_{w,y}(x)=r\wedge rsubst_{w,x}(y)=r$. We recall similarity classes  cf. \cite{Za1}, with the fuzzy membership functions $$\mu_{[x]_{\vartheta}}=r
 \equiv \vartheta(x,y)=r.$$ Then, the family $C=\{[x]_{\vartheta}: x\in U\}$ is a fuzzy partition in the sense of \cite{Za1}.\end{theorem}
 For the proof we refer the reader to \cite{Polk4} or \cite{Polk3}.  Rough subsets turn out to be very useful in formalization of the idea of granulation.
\section{Granular computing}
 It is a very popular nowadays form of computation within the rough set community. The idea of computing with granules of knowledge was put forth in Zadeh \cite{Zade}. Our idea of granulation was given in \cite{PolkC}. A similar approach with usage of templates was proposed in  \cite{Hoa}. For other approaches, see some recent works \cite{Yao2}, \cite{Yao3}, \cite{Lin4}.

\begin{definition} [Granule] \label{grandef}\ Let $rsubst_w$ be a rough subset. Given  $x\in U$  and $r\in [0,1]$, the granule $g_w(x,r)$ is the class $Cls\{y: rsubst_w(y,x,r)\}$.\end{definition}
Cf.  \cite{Polk3} for properties of granules, in particular, see the proof of
\begin{theorem}[Properties of granules]\label{props} For  any Archimedean t-norm $T$, $x, y\in [0,1]$, the transitivity of $T$ - induced rough subsets holds true: $$rsubst_T(x, y, r)\wedge rsubst_T(y, z, s)\supset rsubst_T(x, z, T(r,s))$$ Moreover, each such rough subset relation  is symmetric: $rsubst_T(x, y, r)\equiv rsusbt_T(y, x, r)$.\end{theorem}
\begin{proof} Thm.\ref{props} is a direct consequence of Thm. \ref{arch}.\end{proof} Thm.\ref{props} implies directly the following
\begin{theorem} [Granules as sets] \label{gran1} For a transitive and symmetric rough inclusion $r_subst_w$, each granule $g_w(x,r)$ is identical with the set $\{y: rsubst_w(y,x,r)\}$.\end{theorem}
For the proof cf. \cite{Polk3}.
\begin{theorem} [Properties of granules] \label{granprop} The following are properties of granules in domains of information systems.\end{theorem}
\begin{enumerate}
\item\ $Cmp(y, x)\supset Cmp(y, g_w(x,r))$.
\item\ $Cmp(y, g_w(x,r))\wedge Cmp(z,y) \supset Cmp(z, g_w(x,r))$.
\item\ $rsubst(y,x,r)\supset Cmp(y, g_w(x,r))$.
\item\ $s<r\supset Cmp(g_w(x,r), g_w(x,s))$.
In case of t-norm $T$ - induced transitive rough subset $rsubst_T$:
\item\ $Cmp(y, g_w(x,r))\supset Cmp(g_w(y,s), g_w(x, T(r,s)))$ for each $r, s, x, y\in [0,1]$.
\item\ $(Cmp(y,  g_w(x,r)\wedge 1> r > s)\supset \exists t<1. Cmp(g_w(y,t), g_w(x,s))$.
\end{enumerate}
Proofs can be found in \cite{Polk3}.

\section{Granular deciders} They were introduced in \cite{PolkC} and studied in \cite{PolArt}. They parallel other rough set decision deciders like RIONA  \cite{Riona} or Rseslib 3 \cite{Reslib}. Systems obtained by means of granulation of data are called here granular deciders.
\begin{definition}[Granular decider for a decision system]\label{reflex}\end{definition}
Step 1.\ For a decision system $\mathcal{D}$ = $(U,F,W,d,V_d)$ with decision feature $d$ and its value set $V_d$, $x\in U$,  $r\in \{\frac{1}{|F|}, \frac{2}{|F|}, \ldots,\- 1\}$, and,  $rsubst_w$ a symmetric and transitive rough subset, compute the granule $g_w(x,r)$ for each pair $x, r$.

Step 2.\ Choose a radius $r$ and cover $U$ by computed in Step 1 granules $g_w(x,r)$ for $x\in U$.\ Choose a  non-reducible sub-covering $\mathcal{C}$.

Step 3.\ For each feature $f\in F\cup \{d\}$ ,we define a granular mirror $\overline{f}: {\mathcal{C}}\rightarrow V$ of the feature $f$ by a chosen strategy of voting  for example majority voting MV: with MV, $\overline{f}(g) = MV(\{f(x): x\in g\})$.

Step 4.\ Form the granular decider $\overline{{\mathcal{D}}}$ = $(\mathcal{C}, \{\overline{f}: f\in F\}, V, \overline{d},  V_d)$.

Step 5.\ Apply the chosen decision making protocol to $\overline{{\mathcal{D}}}$.

Step 6.\ Repeat for other values of the radius $r$.

Step 7.\ Output the results for the best radius.\

In \cite{PolArt}, a number of protocols were tested with granular deciders with excellent results: accuracies were higher for some radii than literature results and the size of data was substantially reduced. We give an example.

\begin{example} [Australian credit]\label{credit}\ We show comparative results for the data set Australian Credit \cite{irvine} of 15 features and 690 objects. The best result by any non - rough - set based  method was accuracy = 0.869, coverage not given, by Cal 5, see Statlog \cite{stat}. Best results by rough set methods was by  S. H. Nguyen \cite{Hoa}: accuracy = 0.929 with coverage = 0.623 obtained by {\em simple.templates} method. The best result for coverage = 1.0 was obtained by S. H. Nguyen, loc.cit., with accuracy = 0.875 by {\em tolerance.gen.templ.} method. The best result obtained by the method of granular deciders was accuracy = 0.880 with coverage = 1.0 by the method {\em $8_v1_{variant\ 5}$}, see \cite{PolArt}.\end{example}
We have mentioned the logical content of rough sets. We now illustrate the intensional aspect of rough sets. For intensional logics see \cite{Mon}.
\section{Intensional logic for information/decision systems}\label{inten}

We define rough subset relations $\nu_3$ and $nu_L$ on pairs of sets. For $\nu_L$, see \cite{Polk4}.

\begin{definition}[Intensional rough subsets]\label{rint}

$\nu_3(X,Y,1)$ if and only if $X\subseteq Y$.\

$\nu_3(X,Y,0)$ if and only if $X\cap Y=\emptyset$.\

$\nu_3(X, Y, \frac{1}{2})$ otherwise.\

$\nu_L(X, Y, r)$ if and only if $h(\frac{|X\setminus Y|}{|X|})\geq r$. \end{definition}

Given a set of unary predicates $\mathcal{P}$ interpreted in the domain $U$ of the system $(U,A,V,d, V_d)$, we define intensional logic $IL$ on a system $G$ of granules in $U$ by introducing,  for a given $P$, the intention of $P$ as $\Im(P): G\rightarrow [0,1]$ with extensions at particular granules  $\Im^{\vee}(g,[P])=max\{r: \nu(g, [P],r)\}$. \

We regard a formula $\phi$ as valid if and only if $\nu(g, [\phi],1)$ holds at each $g\in g$. A paraphrase is in order.

\begin{theorem}[Intensional truth]\label{it}\ A formula $\phi$ is true at a granule $g$ if and only if $g\subseteq [\phi]$. In particular, for a decision rule $r: \alpha\rightarrow\beta$, $r$ is true at a granule $g$ if and only if $g\cap [\alpha]\subseteq [\beta]$.\end{theorem}

 \begin{proof} By definition of $\nu$ in Def.\ref{rint}.\end{proof}

 \begin{theorem}[validity]\label{truth}  A formula $\phi$ is valid in the logic $IL$ with respect to a set $G$ of granules if and only if $Cls(G)\subseteq [\phi]$. In case $\bigcup G = U$, this condition comes down to $[\phi]=U$. In the particular case of a decision rule $\eta:\alpha\rightarrow\beta$, $\eta$ is valid if and only if $[\alpha]\subseteq [\beta]$.\end{theorem}

 Properties of truth in the logic $IL$ are collected in the following theorem.

\begin{theorem}[On truth in IL]\label{Tr} The following hold.\end{theorem}

\begin{enumerate}
\item\ If $\phi$ is true at granules $g$ and $h$, then $\phi$ is true at $g\cup h$ and at $g\cap h$.

\item\ If $\phi, \psi$ are true at a granule $g$, then $\phi\vee\psi, \phi\wedge\psi$ are true at $g$.

\item\ A formula $\phi\supset\psi$ is true at $g$ if either $\psi$ is true (then the implication is true for every $\phi$) or both $\phi,\psi$ are true at $g$.
 \end{enumerate}

Concerning a weakening of Thm.\ref{Tr}, we have

 \begin{definition} [Partial truth]\label{Trudeg} For any pair, $g,\nu$, a formula $\phi$ is true at $g$ to a degree of at least $r$ if and only if $\Im^{\vee}(\phi)(g)\geq r$. A formula $\phi$ is false if  $I_{\nu}^{\vee}(g)(\phi)\geq r$ implies $r=0$.\end{definition}

 Graded truth satisfies the following. The symbol $\Rightarrow_L$ denotes the \L ukasiewicz many-valued implication, see \cite{Polk4}.

 \begin{theorem}\label{ontrdeg}\ The following facts hold.\end{theorem}
 \begin{enumerate}
 \item\ A formula $\phi$ is true at $g,\nu$ if and only if $\neg\phi$ is false at $g,\nu$.\
 \item\ $\Im^{\vee}(g)(\neg\phi)\geq r$ if and only if $\Im ^{\vee}(g)(\phi)\geq s \supset s\leq 1-r$.
 \item\ A formula $\phi\supset \psi$ is true at a granule $g$ for $\nu=\nu_3,\nu_L$ if and only if $g\cap [\phi]\subseteq [\psi]$.
\item\ A formula $\phi\supset\psi$ is false at a granule $g$ for $\nu=\nu_3,\nu_L$ if and only if $g\subseteq [\phi]\setminus
 [\psi]$.\
 \item\ for $\nu=\nu_L$, if $\Im^{\vee}(g)(\phi)\geq s$ and $\Im^{\vee}(g)(\psi)\geq t$, then if   $\Im^{\vee}(g)(\phi\supset\psi)\geq r$, then $r\leq s\Rightarrow t$.
     \end{enumerate}

 The Le\'{s}niewski collapse (called also the Le\'{s}niewski erase operation) is the operation of removing from formulae  variables, producing propositional form of the formula. For a formula $\phi$, we denote by the symbol $\phi^*$ the result of the Le\'{s}niewski collapse. Then, (i) $(\neg\phi)^*=\neg (\phi^*)$ (ii)$(\phi\Rightarrow\psi)^*=\phi^*\Rightarrow \psi^*$. When we let the value $[\phi^*]$ of a formula $\phi^*$ at a granule $g$ as $[\phi^*]_g=argmax_r\{\nu_L(g,[\phi^*],r)\}$, then we can write (5) in Thm.\ref{ontrdeg} as $\Im^{\vee}(g)(\phi\supset \psi)\leq [\phi^*]_g\Rightarrow [\psi^*]_g$.\

 \begin{corollary}\label{cortrdeg}\ (i)\ If the decision rule $\phi\supset\psi$ is true at a granule $g$, then the collapse $\phi^*\supset \psi^*$ is true at the granule $g$ (ii)\ if $\phi^*\supset \psi^* <1$ at the granule $g$, then the decision rule $\phi\supset \psi$ is not true at $g$ (iii) a decision rule $\phi\supset\psi$ is true at a granule $g$ if and only if $[\phi^*]_g\leq [\psi^*]_g$.\end{corollary}

\section{Synthesis of compound entities}\label{Cognet}

Cognitive networks are graph structures with nodes possessing computational abilities, examples are neural networks, nature inspired computing machines, see \cite{Kriz}.\ We first discuss networks for processing compound objects.\ We define a compound object as an entity which consists of at least two distinct parts.

 To begin with, we consider  a decision system $\mathcal{D}$ =$(U, F, V, d, V_d)$. We recall the discernibility sets $DIS_F(x,y)=\{f\in A: f(x)\neq f(y)\}$. We define  a new kind of a rough subset.\

 \begin{definition}[Exp-rough subsets]\label{exp} $w_f$ is the weight associated with $f\in F$.\end{definition} $rsubst_{exp}(x,y,r) \equiv exp[-|\sum_{f\in DIS{_F}(x,y)}w_f|^2]\geq r.$

 \begin{theorem} [Exponential  rough subsets are rough subsets]\label{exp1} The relation $rsubst_{exp}$ satisfies  conditions for a rough subset.\end{theorem}

 \begin{proof} We check the conditions (rs1), (rs2). For (rs1), $rsubst_{exp}(x,y,1)\equiv DIS_F(x,y)=\emptyset\equiv IND_F(x,y)\supset Cmp(x,y)$. For (rs2), $DIS_F(x,y)=\emptyset$ implies that for each $z\in U$, $DIS_F(x,z)=DIS_F(y,z)$. \end{proof}
 We collect in the theorem below some properties of exponential  rough subsets.

 \begin{theorem}[Exponential rough subsets: properties]\label{exp2}\ The following can be proved about $rsubst_{exp}$.\end{theorem}
\begin{enumerate}
\item\ $IND_F(x,y)\supset rsubst_{exp}(x,y,1)$.
\item\  $\exists \alpha(r, s). rsubst_{exp}(x,y,r)\wedge rsubst_{exp}(y,z,s)\supset rsubst_{exp}(x,z,\alpha(r,s))$.
\end{enumerate}
 \begin{proof} In view of Thm.\ref{exp1}, property 1 holds true. For property 2, we let $$\alpha(r,s)=r\cdot s\cdot exp\{[2\cdot(log r)\cdot (log s)]^{\frac{1}{2}}\}$$
 Suppose that $rsubst_{exp}(x,y,r)$ and $rsubst_{exp}(y,z,s)$ hold true. Hence,

 $\sum_{f\in DIS_F(x,y)}w_f\leq (-log r)^{\frac{1}{2}}$ and $\sum_{f\in DIS_F(y,z)}w_f\leq (-log s)^{\frac{1}{2}}$.  As

  $DIS_F(x,z)\subseteq DIS_F(x,y)\cup DIS_F(y,z)$, we have $(-log t)^{\frac{1}{2}}\leq (-log r)^{\frac{1}{2}} +  (-log s)^{\frac{1}{2}},$ where $t=sup\{q: rsubst_{exp}(x,z,q)\}$ and we let $t=\alpha(r,s)$. \end{proof}
 From Thm.\ref{exp2}, we obtain a corollary for granules. We denote with the symbol $g_{exp}(x,r)$ any granule defined by means of the exponential rough subset.

 \begin{theorem}[Exponential granules] \label{expgran}  $$x\in g_{exp}(y, r) \wedge y\in g_{exp}(z,s)\supset \exists t. Cmp(g_{exp}(x,t), g_{exp}(y,r))\wedge$$  $$Cmp(g_{exp}(x,t), g_{exp}(z,s))$$ \end{theorem}

 \begin{proof} From the proof of Thm.\ref{exp2} it follows that we have to satisfy the conditions $\alpha(r,t)\geq r$ and $\alpha(s,t)\geq s$ and from these conditions we obtain $t$.\end{proof}
 We can now extend our analysis by considering a model for rough mereological perceptron. We consider an agent $a$ with an information system $(U_a, F_a, V_a)$ which accepts a vector $\overline{x}=[x_1, x_2, \ldots, x_k]$ of entities and applies its {\em fusion operator} $fus_a$ in order to form an output object $\overline{o}=fus_a(\overline{x})$.\

 The agent $a$ is assigned  a set $T_a\subseteq U_a$  of {\em target objects}, $T_a=\{t_i: i\in I\}$. The classification of the output $\overline{o}$ to a target object $t^*$ proceeds according to the rule: $(sel)\ rs_{exp}^*(\overline{o}, t^*)= max\{rs_{exp}^*(\overline{o}, t_i): i\in I\}$.\

The perceptron structure and semantics can be extended to networks of perceptrons $NP$. The network consists of a few layers:  $lay_0$ (input), $lay_1, \ldots,$ $lay_{k-1}, lay_k$ (output). Each layer $lay_i$ consists of agents $a_1^i, a_2^i,  \ldots,\- a_{k_{i}}^i$. Each agent $a_j^i$ has its information system $(U_j^i, F_j^i, V_j^i)$ and a set of targets $T_j^i\subseteq U_j^i$. Entities in each set $U_j^i$ are described in terms of features in feature sets $F_j^i$ and their values in the set $V_j^i$.\

For each agent $a_j^i\in lay_{i}$, the operation $O(a_j^i, b_m^{i+1})$ sends entities at $U_{a_{j}}^i$ into entities at $U_{b_{m}}^{i+1}$ for each agent $b_{m}^{i+1} \in lay_{i+1}$ for $i=0,1,\ldots, k-1$.\

The inner structure of the network $NP$ is designed by {\em designer} $D$. We simplify this design in order to show the essentials. The {\em designer's operator} $D$ secures the coordination among agents in consecutive layers, and among their targets.\

Let $Sel_i$ be the set of all selectors of the form $<x_1^i,x_2^i,\ldots, x_{k_{i}}^i>$ where $x_j^i \in U_j^i$.\ In particular, the subset $Sel_T^i$ of $Sel_i$ is the set of all selectors in which entities are exclusively targets.\

For each selector $\sigma = <x_1^i,x_2^i,\ldots, x_{k_{i}}^i> \in Sel_i$ and each agent $b_m^{i+1}$, the entity constructed at $U_m^{i+1}$ is the vector $$O(\sigma) = fus_{b_{m}}^{i+1}([O(a_1^i, b_m^{i+1})(x_1^i), \ldots, O(a_{k_{i}}^i, b_m^{i+1})(x_{k_{i}})])$$
The requirement of coordination runs as follows.
\begin{definition} [Coordination of the network]\label{coord} The following  are to be observed.\end{definition}
\begin{enumerate}
\item\ For each agent $b_m^{i+1}\in lay_{i+1}$, the feature set $F_m^{i+1}$ is the union of feature sets $F_1^i, F_2^i,\ldots, F_{k_{i}}^i$. i.e., we assume the 'coordinate - wise' assembling of compound entities. We assume that sets $F_j^i$ are pair-wise disjoint.
\item\ For each $x\in U_m^{i+1}$ there exists a selector $\sigma\in Sel_i$ such that $x= O(\sigma)$; in particular, for each target $t_m^{i+1}$, there exists a selector $\sigma_T  \in Sel_T^i$ such that $t = O(\sigma_T)$.
\item\ For each selector $\sigma_T$, the entity $O(\sigma_T)$ is a target.
\item\ If $r_i =sup\{r: rsubst_{exp}(x_j^i,t_j^i,r)\}$ for $x_i, t_i \in U_j^i$, any $i,j$,  $\overline{x}=[x_1^i, x_2^i,  \ldots,$
$x_{k{i}}^i]$, and, $t=O(<t_1^i, t_2^i,  \ldots, t_{k{i}}^i>)$, then  $rsubst_{exp}(O(\overline{x}), t,r^*)$ with $r^*\geq max\{r_i\}$.
\end{enumerate}
The condition (4) is feasible in the light of our assumption about 'coordinate - wise' structure of the network in condition (1).\

The upshot of designer's requirements is that compound entities at the next layer are described by means of features of their parts sent by agents in the preceding layer.\ From condition (4), we extract a definition of the semantic functor on degrees of containment.
\begin{definition} [Semantic functor]\label{sem} In the notation of condition (4) in Def.\ref{coord}, we define the functor $\Psi$   by letting $\Psi(\overline{r})=r^*$, where $\overline{r}=[r_1, r_2,  \ldots, r_{k{i}}]$. Also by (4) in Def. \ref{coord}, $r^*\geq r_i$ for each $r_i$.\end{definition}
Given the set of targets $T_{output}^k$, and an input $\overline{x}$, we find for each component $x_i$ of the input vector, the closest target $t^0_i \in T_i^0$ with $rsubst_{exp}(x_i, t_i^0, r_i)$, i.e., $x_i \in g_{exp}(t_i^0, r_i)$. By condition (4) in Def.\ref{coord} and by means of the semantic operator $\Psi$ of Def.\ref{sem}, we find $t_j^{1,*} \in T_j^1$ such that $rsubst_{exp}(O(\overline{x}_j, t_j^1, r_j^{1,*})$ holds true with $r_j^{1, *} \geq sup\{r_i\}$.

Iterating this step, we find the sequence of sets of targets , the sequence of sets of entities initiated by $\overline{x}$, sequences of sets of increasing degrees $\{r_1^{i,*}, r_2^{i,*}, \ldots,$

\noindent $r_{n{i}}^{i,*}\}$, and finally the entity $y \in U_{output}^k$ and a target $t_k^{j,*}$ with the degree $r_{output}^*$, the largest of all degrees in the sequences. In this way, we classify the input to the closest target.\

Many agent - systems are an abstract rendition of the idea of networking which does encompass problem solving  in machine learning cf. \cite{Diet}. The networking idea is based on fusion operations which made ways into agents signatures.

\begin{definition} [Granular agent]\label{grana} The signature of a granular agent $a$ is $<U_a, rsubst_a, L_a, fus_{rsubst_{a}}, fus_{l_{a}}, fus_{e_{a}}>,$ where \end{definition}
\begin{enumerate}
\item\ $U_a$ is the universe of entities manipulated by $a$.
\item\  $rsubst_a$ is the   rough subset relation used by $a$.
\item\ $L_a$ is the logical language used by $a$.
\item\  $fus_{rsubst_{a}}$ is the fusion operator showing how $a$ relates $rsubst_a$ to rough subsets applied by  agents connected to $a$.
\item\ $fus_{l_{a}}$ is the logical synthesis operator showing how logical formulae at agents connected to $a$ are transformed into a formula in $L_a$.
\item\ $fus_{e_{a}}$ is the fusion operator showing how entities at agents connected to $a$ are forming the compound entity  in $U_a$.
\end{enumerate}
We  consider as an illustration the  elementary case  when the agent $a$ receives information about rough subsets $rsubst_b, rsubst_c$ entities $x_b, x_c$, and,  formulas $\phi_b, \phi_c$ submitted  by agents $b,c$ connected to $a$ as leaves of the tree with the root  $a$. A generalization is straightforward.\

We assume that  operator $fus_{e_{a}}$ acts on  pairs of objects: $x_b$ sent by the agent $b$ and $x_c$ sent by the agent $c$ in the coordinate - wise way, i.e., $fus_{e,{a}}(x_b,x_c)=x_bx_c$.\ We may assume that $x_b, x_c$ are binary codes for parts sent to $a$ from $b, c$ and then $x_bx_c$ is the concatenation of codes. Some quite obvious and securing regularity in networks conditions may be listed.

\begin{enumerate}
\item[(RC1)]\ $Cmp_b(x_b, x'_b), Cmp_c(x_c, x'_c)\supset Cmp_a(fus_{e_{a}}(x_bx_c),fus_{e_{a}}(x'_bx'_c))$.
\item[(RC2)]\ The symbol of double turnstile in $x\models \phi$ means that  $x$ satisfies $\phi$.  Then: if $x_b\models \phi_b$ and $x_c \models \phi_c$, then $fus_{e_{a}}(x_b, x_c)\models fus_{l_{a}}(\phi_b, \phi_c)$.
\item[(RC3)]\ If $rsubst_b(x_b, y_b, r_b)$ and $rsubst_c(x_c, y_c, r_c)$, then $$rsubst_{a}(fus_{e_{a}}(x_b, x_c), fus_{e_{a}}(y_b, y_c), fus_{rsubst_{a}}(r_b, r_c))$$
\item[(RC4)]\ We define indirectly the operator $fus_{g_{a}}$ on granules to be applied in (RC5):

 If $Cmp_b(x_b, g_{rsubst_{b}}(y_b, r_b))$ and $Cmp_c(x_c, g_{rsubst_{c}}(y_c, r_c))$, then $$Cmp_a(fus_{e_{a}}(x_b, x_c),[ fus_{g_{a}}(g_{rsubst_{b}}(y_b, r_b),  g_{rsubst_{c}}(y_c, r_c))$$ $$(fus_{e_{a}}(y_b, y_c), fus_{rsubs_a}(r_b, r_c))]$$
\item[(RC5)]\ If $\Im_b^{\vee}(\phi_b)(g_b)\geq r_b$ and $\Im_c^{\vee}(\phi_c)(g_c)\geq r_c$, then

$\Im_a^{\vee}(fus_{l_{a}}(\phi_b, \phi_c)(fus_{g_{a}}(g_b,g_c)\geq fus_{rsubst_{a}}(r_b,r_c).$
\end{enumerate}
We return to our case of the tree rooted at $a$ with leaves $b, c$.

\begin{theorem} [Rules of propagation]\label{propagation} Operators $fus_{rsubst_{a}}$ and $fus_{g_{a}}$ are ruled by the \L ukasiewicz t-norm  $L$ whereas extensions  $\Im^{\vee}$ are ruled by the product t-norm.\end{theorem}
\begin{proof}\ We indicate the direction of proof. We consider our tree with the root $a$ and leaves $b,c$. By our assumptions:
\begin{enumerate}
\item\ $F_a=F_{b}\times F_{c}$.
\item\ $fus_{e_{a}}(x_b,x_c)=x_bx_c$.\
\item\ $IND_a(fus_{e_{a}}(x_b,x_c), fus_{e_{a}}(y_b,y_c))= IND_b(x_b,y_b)\cup IND_c(x_c,y_c)$.
\item\ $DIS_a(fus_{e_{a}}(x_b,x_c), fus_{e_{a}}(y_b,y_c))\subseteq DIS_b(x_b,y_b)\times F_c \cup F_b\times DIS_c(x_c,y_c)$.
\item\ $ rsubst_a(x,y)= 1 - \frac{|DIS_a(x,y)|}{|F_a|}$.
\item\ $r=max\{s: (fus_{e_{a}}(x_b,x_c), fus_{e_{a}}(y_b, y_c), s)$ =

$1- $ $\frac{|DIS_a(fus_{e_{a}}(x_b,x_c), fus_{e_{a}}(y_b,y_c))}{|F_b|\cdot |F_c|}$.
\item\ $r\geq L(max \{s: m_b(x_b,y_b,s)\}, max\{t: m_c(x_c,y_c,t)\})$.
\item\ Granules propagate along the line:

$fus_{g_{a}}(g_b(x_b, r_b), g_c(x_c, r_c)) =  g_a(fus_{e_{a}}(x_b, x_c), L(r_b, r_c))$.
\item\ For logical formulas: $fus_{l_{a}}(\phi_b,\phi_c)=\phi_b\wedge\phi_c$. Extensions propagate thus along the lines
$\Im_a^{\vee}(fus_{l_{a}}(\phi_b,\phi_c)(fus_{g,a}(g_b,g_c)= \Im_b^{\vee}(\phi_b)(g_b)\cdot \Im_c^{\vee}(\phi_c)(g_b)$.
\end{enumerate}
\end{proof}

Our last topic is about spatial reasoning and its applications to mobile robotics.

\section{Spatial reasoning and  robotics}\label{spares}

In this part, we define a mereogeometry and apply it towards strategies for navigation for robots and their formations. First, we revisit the notion of betweenness relation going back to  \cite{Tar2} and an extension from the plane to Euclidean spaces in cite{vB}.
\begin{definition}[Distance function $\rho$]\label{rho} Given a rough subset $rsubst$, we define the function $\rho(x, y)=rs^*(x,y)+rs^*(y,x)$.\end{definition}
$\rho$ is a pseudo-metric, lacking possibly the triangle condition.\ Observe the inversion: $\rho(x,y)=1 equiv x=y$.

We model obstacles in robot' space as well as safety regions about robots as closed rectangles with sides parallel to coordinate axes in the Euclidean plane.

We define the relation of rough nearness, $rnear$ which is a paraphrase of  the nearness relation in \cite{vB}.
\begin{definition}[r-nearness]\label{near} We say that the instance $rnear(x,y,z)$ holds true if and only if $\rho(x,y)\geq \rho(x,z)$. \end{definition}
We define the relation $rbtw$ (r-betweenness)\cite{Tar2}, \cite{vB}.
\begin{definition} [r-betweenness]\label{betwee} We say that entity $z$ is between  $x$ and $y$, if and only if for each $w$ either $w=z$ or  $rnear(z,x,w)$ or $rnear(z,y,w)$. \end{definition}
\begin{definition}[Extent]\label{ext}  For disjoint rectangles $x,y$, the extent $ext(x,y)$ is the smallest rectangle containing $x$ and $y$. See \cite{Polk4}.\end{definition}
A characterization of betweenness is as follows.
\begin{theorem}[Only sum is between]\label{only} For each pair $x,y$ of things, the only thing between $x$ and $y$ is the sum $x+y$.\end{theorem}

\begin{proof} Assume that $w\neq x+y$, hence, by the definition of $+$, there exists a thing $u$ such that $part(u,w)$ and either $ext(u,x)$ or $ext(u,y)$. For $t$ either $x$ or $y$, depending on the latter case, we obtain contradiction with the notion of betweenness.\end{proof}
\begin{remark}\label{rem} In case of closed rectangles  with sides parallel  to coordinate axes $x,y$, the mereological sum $x+y$ is the extent $ext(x, y)$, i.e, the smallest rectangle containing both $x, y$. In this case,  $z$ is between  $x, y$ if and only if  $Cmp(z, ext(x,y))$ holds cf. \cite{Polk4}.\end{remark}
On the basis of Remark ref{rem}, we may extend the notion of r-betweenness to all proper rectangles which are subsets of $extent(x, y)$ disjoint to $x, y$. In this way, we open the mereogeometry environment to mobile robot navigation.
The principal component of our approach is the mereological potential field, the counterpart of  potential fields in robotics see  \cite{Khab}, \cite{Osm}.
 The crucial notion  in navigation problem for formations of robots is the r-betweenness relation of Def. \ref{betwee}.
\begin{definition}[Robot formations]\label{r-formation}\  An r-formation is a pair $(F, B)$,  where $F$ is a finite set of robots and $B$ is an r-betweenness relation on $F$.\end{definition}
\begin{example}\label{nest}\ We give two lines of the  description of a cross formation of five robots, see \cite{Osm} (Roomba is the trademark of IRobot Inc).\

 (cross
   (set
     (max-dist 0.25 roomba 0 (between roomba 0 roomba 1 roomba 2))\

      ........................

       (not-between roomba 4 roomba 1 roomba 2)\
      )
)
\end{example}
To keep the formation within limits a parameter $max-dist$ is applied which sets the maximal distance among robots. In Fig.1, we show the result of navigating a cross formation of robots \cite{Osm}.

 \begin{figure}[!h]
  \vspace*{-3mm}
  \centering\includegraphics[width=9.1cm]{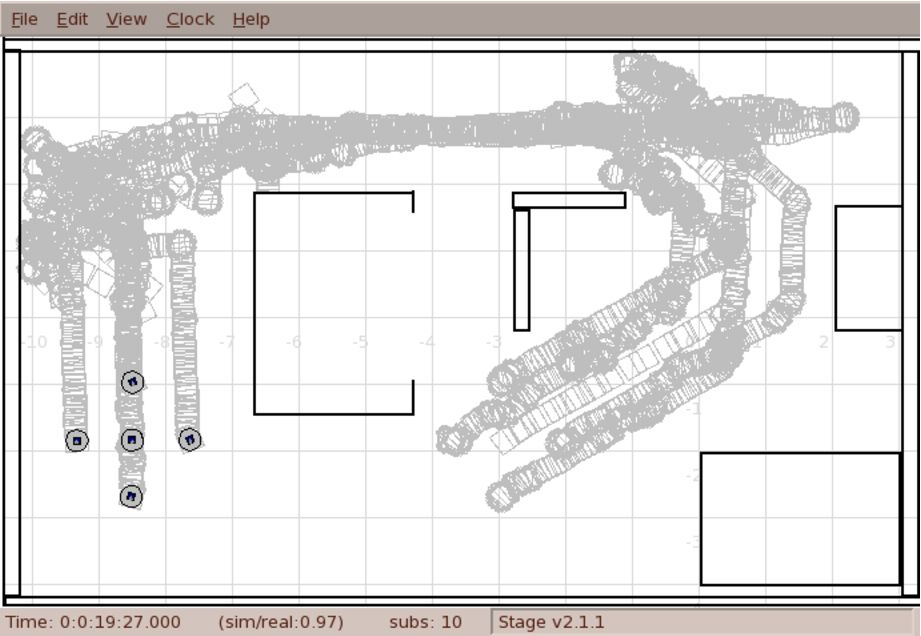}\vspace*{-2mm}
  \caption{The cross formation navigating obstacles. Courtesy PJIIT Publishers}
  \end{figure}

  {\bf A discussion and conclusions}\

  The aim of this work has been to show in a single  exposition the results of the extension to rough mereology of the theory of mereology This theory treats concepts in a holistic way allowing for relations among concepts in terms of parts, hence, it provides a theoretical foundation for rough set theory which  is based on holistic containment of exact vs. inexact sets.  We describe the extension of mereology to rough mereology, which submits to  mereology the notion of uncertain measurable containment.  We show how to use rough mereology for introduction of intensional logics into rough set theory decision systems, for definitions of  abstract versions of fuzzy reasoning, for introduction of mechanisms for granulation in granular decision systems, for application of granular deciders to classification tasks, for reasoning in multi-agent systems and in cognitive networks, and finally, in problems of intelligent robotics by creating by means of mereogeometry  mereological potential fields and navigate with formations of robots.
  Our approach gives flexibility o the problems of robot navigation and provides better classification and decision tools. We do hope that ideas presented in this article will gain popularity and will bring new ideas and tools.\

 {\bf  Declaration of Competing Interest}\

The author declares that there is  no known competing financial interests or personal relationships that could have appeared to influence the work reported in this paper.

\end{document}